\documentclass[10pt,journal,compsoc]{IEEEtran}
\usepackage{algorithm}
\usepackage{float}
\usepackage{algpseudocode}
\usepackage{placeins}
\usepackage{graphicx}
\usepackage{threeparttable}
\usepackage{booktabs}        
\usepackage{array}           
\usepackage{amsmath,amssymb,amsfonts} 
\usepackage{amsthm}           
\usepackage{bm}               
\usepackage{caption}          
\usepackage{subcaption}       
\usepackage{multirow}         
\usepackage{multicol}         
\usepackage{tabularx}         
\usepackage{xcolor}           
\usepackage{cite}             
\usepackage{textcomp}         
\usepackage{float}            
\usepackage{enumitem}         
\usepackage{comment}          
\usepackage[bottom]{footmisc} 
\usepackage{hyperref} 
\IEEEoverridecommandlockouts  
\newtheorem{theorem}{Theorem}
\newtheorem{definition}{Definition}

\theoremstyle{remark}   
\newtheorem{remark}{Remark}
\title{NeuroGame Transformer: Gibbs-Inspired Attention Driven by Game Theory and Statistical Physics}

\author{
    \IEEEauthorblockN{Djamel Bouchaffra\IEEEauthorrefmark{1}, 
                      Faycal Ykhlef\IEEEauthorrefmark{2}, 
                      Hanane Azzag\IEEEauthorrefmark{3}, 
                      Mustapha Lebbah\IEEEauthorrefmark{1}, 
                      Bilal Faye\IEEEauthorrefmark{3}}
    \\
    \IEEEauthorblockA{\IEEEauthorrefmark{1}
    Paris-Saclay University, UVSQ, DAVID Lab, 78035, Versailles, France} \\
    \IEEEauthorblockA{\IEEEauthorrefmark{2}
    TSAM, ASM, Centre for Development of Advanced Technologies, Algiers, Algeria} \\
    \IEEEauthorblockA{\IEEEauthorrefmark{3}
    LIPN, UMR CNRS 7030, Sorbonne Paris Nord University, Villetaneuse, France} \\
    \small{\{djamel.bouchaffra, mustapha.lebbah\}@uvsq.fr, fykhlef@cdta.dz, 
    \{bilal.faye, azzag\}@lipn.univ-paris13.fr}
}

\algrenewcommand\algorithmiccomment[1]{\hfill \(\triangleright\) #1}

\begin{document}

\maketitle

\begin{abstract}
Standard attention mechanisms in transformers are limited by their pairwise formulation, which hinders the modeling of higher-order dependencies among tokens. We introduce the NeuroGame Transformer (NGT) to overcome this by reconceptualizing attention through a dual perspective: tokens are treated simultaneously as players in a cooperative game and as interacting spins in a statistical physics system. Token importance is quantified using two complementary game-theoretic concepts—Shapley values for global, permutation-based attribution and Banzhaf indices for local, coalition-level influence. These are combined via a learnable gating parameter to form an external magnetic field, while pairwise interaction potentials capture synergistic relationships. The system's energy follows an Ising Hamiltonian, with attention weights emerging as marginal probabilities under the Gibbs distribution, efficiently computed via mean-field equations. To ensure scalability despite the exponential coalition space, we develop importance-weighted Monte Carlo estimators with Gibbs-distributed weights. This approach avoids explicit exponential factors, ensuring numerical stability for long sequences. We provide theoretical convergence guarantees and characterize the fairness--sensitivity trade-off governed by the interpolation parameter. Experimental results demonstrate that the NeuroGame Transformer achieves strong performance across SNLI, and MNLI-matched, outperforming some major efficient transformer baselines. On SNLI, it attains a test accuracy of 86.4\% (with a peak validation accuracy of 86.6\%), surpassing ALBERT‑Base and remaining highly competitive with RoBERTa‑Base. Code is available at \url{https://github.com/dbouchaffra/NeuroGame-Transformer}.
\end{abstract}

\vspace{0.1cm}
\section{Introduction}
\noindent Transformers have revolutionized natural language processing through attention mechanisms that compute context-dependent token representations. However, conventional attention relies primarily on pairwise similarity, limiting its ability to capture higher-order semantic dependencies involving multiple tokens. Recent advances explore higher-order attention~\cite{AlersValentin2023NeuroSymbolic}, kernel-based reformulations~\cite{choromanski2021rethinking}, linear approximations~\cite{katharopoulos2020transformers}, state-space models~\cite{gu2023mamba}, and dynamic computation~\cite{dehghani2023universal}—developments critical for long-context domains such as genomics~\cite{genome2026}, high-resolution vision, and video understanding~\cite{feng2025breaking}.

We propose a fundamentally different perspective bridging \textit{cooperative game theory} and \textit{statistical physics} to reconceptualize transformer attention~\cite{Bouchaffra2025-NN,Bouchaffra2026-IEEETC}. Instead of pairwise similarity, we treat tokens as both strategic players in a cooperative game and interacting spins in a thermodynamic system. This dual interpretation enables modeling higher-order semantic structure through token coalitions. Concretely, we introduce a \emph{semantic coalition game} over token subsets: each coalition $C$ is assigned a scalar value $v(C) = -E(C)$, where $E(C)$ is a physical energy. Coherent coalitions receive lower energy and thus higher activation. \textit{Shapley values} capture globally fair contributions, while \textit{Banzhaf power indices} identify locally decisive tokens; we combine them into an \textit{external magnetic field} that balances global and local sensitivity. Pairwise interaction potentials capture synergistic or antagonistic token relationships, and attention weights are derived from the Gibbs (Boltzmann) distribution over spin configurations, computed via mean-field equations to avoid the intractable sum over $2^n$ states.

To compute Shapley values, Banzhaf indices, and interaction potentials efficiently, we employ importance-weighted Monte Carlo estimation with \emph{Gibbs-distributed importance weights}. This avoids explicit $2^n$ factors, yields unbiased estimates, and makes the partition function cancel naturally. Building on these principles, we introduce the \textit{NeuroGame Transformer}, which makes the following contributions:
\begin{itemize}
\item We define a coalitional semantic game with an energy-based characteristic function, enabling higher-order semantic modeling grounded in statistical physics.
\item We develop scalable Monte Carlo estimators for Shapley values, Banzhaf indices, and interaction potentials using Gibbs-distributed importance weights, achieving \textit{linear complexity $\mathcal{O}(K \cdot n)$} instead of exponential $\mathcal{O}(2^n)$, where $K$ is the number of sampled coalitions.
\item We derive a NeuroGame attention mechanism combining Shapley and Banzhaf values via a learnable context-dependent parameter, with pairwise potentials capturing token interactions.
\item We compute attention weights via mean-field equations, providing a principled mechanism uniting coalition-based reasoning and energy-based modeling.
\item We demonstrate empirically that the NeuroGame Transformer outperforms standard and efficient transformer baselines on natural language inference benchmarks, yielding interpretable coalition-aware attention patterns.
\end{itemize}
The NeuroGame Transformer departs fundamentally from conventional attention by evaluating each token's marginal contribution across all coalitions using game-theoretic principles and deriving attention weights from the Gibbs distribution of a statistical physics system. This coalition-aware perspective, made tractable through Gibbs-based Monte Carlo estimation and mean-field approximation, enables richer semantic modeling and explicit interpretability.
\section{Related Work}
\subsection{Standard and Efficient Attention} The transformer architecture~\cite{vaswani2017attention} relies on pairwise dot-product attention. Efficient variants (e.g., sparse~\cite{draye2025sparse}, linearized~\cite{katharopoulos2020transformers,choromanski2021rethinking}) improve scalability but preserve the pairwise similarity paradigm.
\subsection{Game-Theoretic Methods in Machine Learning}
Shapley values~\cite{shapley1953value} and Banzhaf indices~\cite{banzaf1965weighted} are widely used for post-hoc feature attribution~\cite{lundberg2017shap,sundararajan2017axiomatic} and data valuation~\cite{ghorbani2019data}. Their integration as a core attention mechanism remains unexplored.
\subsection{Statistical Physics and Energy-Based Models}
The softmax operation can be viewed as a Boltzmann distribution over energies~\cite{zhang2018graph,du2020energy}, with temperature controlling sharpness~\cite{guo2017calibration}. Energy-based perspectives have informed variational formulations~\cite{liu2021energy} and softmax alternatives~\cite{luo2021stable}.

\subsection{Our Positioning}
The NeuroGame Transformer uniquely synthesizes three traditionally separate directions. Unlike efficient transformers that approximate pairwise similarity, we model higher-order semantic interactions through token coalitions. Unlike game-theoretic methods that apply Shapley and Banzhaf values post-hoc, we integrate them directly into attention computation. Unlike energy-based models that define arbitrary energy functions, we ground ours in cooperative game theory via the Ising Hamiltonian. Our approach is constructive rather than descriptive: we build attention from first principles using game-theoretic valuation and statistical physics, yielding an expressive, inherently interpretable, and theoretically grounded mechanism.

\section{NeuroGame Transformer: Fusion of Game Theory and Statistical Physics}
\subsection{Notation and Preliminaries}
Consider a sequence of $n$ tokens $\mathcal{T} = \{t_1, t_2, \ldots, t_n\}$, where each token $t_i$ is associated with a contextual embedding $\mathbf{x}_i \in \mathbb{R}^d$. We denote the power set of all possible token coalitions by $2^{\mathcal{T}}$, which contains $2^n$ elements.

\begin{definition}[Token Coalition]
When tokens are viewed as players in a cooperative game, a \emph{coalition} is defined as any subset
$C \subseteq \mathcal{T}$. 
\end{definition}
The set of all possible coalitions is given by the power set $2^{\mathcal{T}}$. Each coalition represents a group of tokens whose joint presence contributes collectively to the model's
representation or prediction. This contribution is quantified by a characteristic function: $v : 2^{\mathcal{T}} \rightarrow \mathbb{R}$, where $v(C)$ measures the semantic, predictive, or energetic contribution of the token subset $C$. For a token $t_i \in \mathcal{T}$ and a coalition $C \subseteq \mathcal{T} \setminus \{t_i\}$, the marginal
contribution of $t_i$ to $C$ is:
\begin{equation}\label{contribution}
\Delta_i v(C) = v(C \cup \{t_i\}) - v(C).
\end{equation}

\subsection{The Coalitional Semantic Game}
We define a \textit{semantic characteristic function} $v: 2^{\mathcal{T}} \rightarrow \mathbb{R}$ that assigns to each coalition $C \subseteq \mathcal{T}$ a scalar value representing its semantic importance within the context. This function satisfies:
\begin{enumerate}
    \item \textit{Normalization:} $v(\emptyset) = 0$
    \item \textit{Monotonicity:} For any $C \subseteq T \subseteq \mathcal{T}$, $v(C) \leq v(T)$
    \item \textit{Boundedness:} There exists $M > 0$ such that $|v(C)| \leq M$ for all $C \subseteq \mathcal{T}$
\end{enumerate}

\subsection{Energy-Based Semantic Characteristic Function]}
Let $\mathcal{T} = \{t_1, \dots, t_n\}$ denote the set of tokens in a sequence.
At a given transformer layer, each token $t_i$ is represented by a
\emph{position-aware embedding}
$\mathbf{x}_i = \mathbf{e}_i + \mathbf{p}_i \in \mathbb{R}^d$,
where $\mathbf{e}_i$ denotes the token embedding and $\mathbf{p}_i$ the positional encoding.
\begin{definition}[Energy-Based Semantic Characteristic Function]
For any coalition $C \subseteq \mathcal{T}$, we define its
\emph{semantic activation potential} as:
\begin{equation}
v(C) = f\Bigg(\Big\| \sum_{t_i \in C} \mathbf{W}_v \mathbf{x}_i \Big\|_2 \Bigg),
\end{equation}
where $\mathbf{W}_v \in \mathbb{R}^{d \times d}$ is a learned projection matrix and
$f$ is a mild nonlinearity (e.g., $\mathrm{ReLU}$ or $\tanh$).
\end{definition}
We interpret $v(C)$ as a coalition-level activation or interaction strength.
In accordance with statistical physics, the corresponding coalition energy
is defined as follows: $E(C) = -v(C)$, which assigns lower energy to more coherent or strongly interacting coalitions. This energy formulation naturally induces a Gibbs (Boltzmann) distribution $D_\gamma$
over token coalitions:
\begin{equation}
\mathcal{D}_\gamma(C) \propto \exp\!\Big(-E(C)/\gamma\Big)
= \exp\!\Big(v(C)/\gamma\Big),
\end{equation}
where $\gamma > 0$ is a temperature parameter controlling the concentration of probability mass over coalitions.
\subsection{Global and Local Tokens Contribution}
we define in this section two semantic measures: The \textit{global contribution} conveyed by Shapley value, and the \textit{local semantic gain variation} conveyed by Banzhaf power index.
\subsubsection{Shapley Value: Global Semantic Contribution}
The Shapley value measures the \emph{average marginal semantic contribution} of a token over all possible \emph{permutations} of token arrivals, ensuring fairness, symmetry, and permutation invariance. Let $\Pi(\mathcal{T})$ denote the set of all $n!$ permutations of the token set $\mathcal{T}$. For token $t_i$, let $\mathcal{P}_i(\pi)$ be the set of tokens that appear before $t_i$ in permutation $\pi \in \Pi(\mathcal{T})$. The Shapley value of $t_i$ is then defined as:
\begin{equation}\label{shapley}
\phi_i(v) =
\frac{1}{n!} \sum_{\pi \in \Pi(\mathcal{T})} 
\Big[ v\big(\mathcal{P}_i(\pi) \cup \{t_i\}\big) - v\big(\mathcal{P}_i(\pi)\big) \Big].
\end{equation}
\begin{remark}
Since equation~\ref{shapley} involves an expectation over the factorial-sized permutation space $\Pi(\mathcal{T})$, exact evaluation of the Shapley value is computationally infeasible, 
motivating the use of importance-weighted Monte Carlo estimation.
\end{remark}
In the case $\sum_{j=1}^n \phi_j(v)\neq 0$, the normalized Shapley value is expressed as: $
\phi_i^\text{norm}(v) =
\frac{\phi_i(v)}{\sum_{j=1}^n \phi_j(v)}$.

\subsubsection{Banzhaf Power Index: Local Semantic Gain Variation}
The Banzhaf power index evaluates the \emph{sensitivity of semantic gain} to the presence or absence of a token across coalitions, treating all coalitions as equally likely.  
Unlike the Shapley value, it does \emph{not} consider permutations of token order and emphasizes local variations in contribution. The raw Banzhaf index of token \( t_i \) is defined as:
\begin{equation}\label{banzhaf}
\beta_i(v) =
\frac{1}{2^{n-1}}
\sum_{C \subseteq \mathcal{T} \setminus \{t_i\}}
\Big[
v(C \cup \{t_i\}) - v(C)
\Big].
\end{equation}
\begin{remark}
Likewise, as equation~\ref{banzhaf} involves a summation over all $2^{n-1}$ coalitions, 
exact computation of the Banzhaf power index is computationally intractable, 
motivating the use of importance-weighted Monte Carlo estimation.
\end{remark}
In the case $\sum_{j=1}^n \beta_j(v)\neq 0$, the normalized form is expressed as: $\beta_i^\text{norm}(v) =\frac{\beta_i(v)}{\sum_{j=1}^n \beta_j(v)}$.

\subsection{Token-Level Importance Score}
We define, in the context of the NeuroGame transformer, a token-level importance score
derived from coalitional game-theoretic measures, and subsequently map this importance
to an external field of a Gibbs energy function.

\begin{definition}[Token Importance as External Field]
The importance of token $t_i$, denoted by $J_i$, is defined as a context-aware convex combination of
the normalized Shapley value and the normalized Banzhaf power index:
\begin{equation}
J_i = \lambda_i \, \phi_i^{\mathrm{norm}}(v) + (1-\lambda_i)\, \beta_i^{\mathrm{norm}}(v),
\end{equation}
where: 
\begin{equation}
\lambda_i=\sigma(w^Tx_i+b).
\end{equation}
The parameter $\lambda_i$ depends on the token representation $x_i\in \mathbb{R}^d$, the weight vector $w$, the bias $b\in\mathbb{R}$ and the sigmoidal function $\sigma$ that maps the score to $(0,1)$. 
\end{definition}

The model learns when marginal contribution (Shapley) matters $(\lambda_i\approx 1$: Shapley-dominant) and when independent swing power (Banzhaf) matters $(\lambda_i\approx 0$: Banzhaf-dominant) or a balanced regime $(\lambda_i\approx 0.5)$. The quantity $J_i$ measures the contribution and relevance of token $t_i$ within the cooperative game defined by $v(C)$, with larger values indicating higher importance.

\subsection{Token Pairwise Interaction}
Likewise, we define a pairwise interaction between tokens representing the \textit{synergistic} or \textit{antagonistic} effect between two tokens when they are part of the same coalition. For any two distinct tokens $t_i, t_j \in \mathcal{T}$, their \textit{interaction potential} measures the synergistic or antagonistic effect when they appear together in a coalition, beyond what their individual contributions would predict. For a given context coalition $C \subseteq \mathcal{T} \setminus \{t_i, t_j\}$ that contains neither token, the conditional interaction is:
\begin{equation}
\Delta_{ij}(C) = v(C \cup \{t_i, t_j\}) - v(C \cup \{t_i\}) - v(C \cup \{t_j\}) + v(C).
\end{equation}
\begin{definition}[Token Pairwise Interaction]
The overall interaction potential $J_{ij}$ is then defined as the expected value of $\Delta_{ij}(C)$ over all possible contexts (subsets of $C$), weighted appropriately:
\begin{equation}
J_{ij} = \sum_{C \subseteq \mathcal{T} \setminus \{t_i, t_j\}} w(C) \, \Delta_{ij}(C),
\end{equation}
where $w(C)$ is a uniform weighting scheme.
\end{definition}

\subsection{Ising Model-based Token System Total Energy}
We now provide the total energy (or Hamiltonian) of a configuration system (an entire input sequence of tokens). However, before expressing of the total energy configuration system, we make the following clarification:
\begin{definition}[Coalition]
A coalition, denoted by $C \subseteq \mathcal{T}$, is any subset of tokens from the token set $\mathcal{T} = \{t_1, t_2, \ldots, t_n\}$. The characteristic function $v(C)$ assigns a value to each coalition.
\end{definition}

\begin{definition}[Spin Configuration]
A spin configuration system which is denoted by $S = (s_1, s_2, \ldots, s_n) \in \{-1, +1\}^n$, is a complete assignment of spin states to every token in the sequence $C$, where $s_i$ represents the activation state of token $t_i$.
\end{definition}

\begin{definition}[Token Energy]
The total energy of the token system $S$ is defined as:
\begin{equation}\label{Hamilton}
H(S)= -\sum_{i=1}^{n} J_i s_i-\sum_{i=1}^{n-1}\sum_{j=i+1}^{n}J_{ij} \, s_i s_j.
\end{equation}
\end{definition}
The variables $s_i \in \{-1, +1\}$ represents the ``activation state'' or ``spin'' of token $t_i$. These activation states are estimated using the mean-field equation~described in section~\ref{mean-field}.
\begin{definition}[Interpretation of Total Energy]
The Hamiltonian $H(S)$ encapsulates the complete energetic state of the token configuration system $S$:
\begin{itemize}
    \item \textbf{First term ($-\sum_i J_i s_i$):} represents the external field contribution. Tokens with high importance $J_i$ (positive external field) are energetically favored to be in the $s_i = +1$ (active) state, as this lowers the total energy. Conversely, tokens with low importance $J_i$ are favored to be inactive ($s_i = -1$).
    
    \item \textbf{Second term ($-\sum_{i=1}^{n-1}\sum_{j=i+1}^{n}J_{ij} \, s_i s_j$)}: represents the pairwise interaction contribution. 
    \begin{itemize}
        \item If $J_{ij} > 0$ (synergistic interaction), the energy is minimized when $s_i$ and $s_j$ have the same sign ($+1,+1$ or $-1,-1$), encouraging \textbf{cooperative activation or deactivation}.
        \item If $J_{ij} < 0$ (antagonistic interaction), the energy is minimized when $s_i$ and $s_j$ have opposite signs, encouraging \textbf{competitive behavior} where one token activates while the other suppresses.
        \item If $J_{ij} = 0$, tokens $i$ and $j$ are independent.
    \end{itemize}
\end{itemize}
The system naturally seeks low-energy configurations, which correspond to coherent token activations that maximize overall semantic coherence.
\end{definition}

\subsection{Attention Weights}
We now provide a definition of attention weights assigned to each token within a token system configuration:
\begin{definition}[Coalitional Attention Weight]
The attention weight $\alpha_i$ of token $t_i$ is defined as the marginal probability that token $t_i$ is in the active state ($s_i = +1$) under the Gibbs distribution over all possible token system configurations $S$:

\begin{equation}
\alpha_i = P(s_i = +1) = \sum_{S: s_i = +1} P(S),
\label{eq:attention_marginal}
\end{equation}
where $P(S)$ is the Gibbs distribution:

\begin{equation}
P(S) = \frac{\exp\left(-H(S)/\gamma\right)}{\displaystyle\sum_{\mathbf{S}'} \exp\left(-H(S')/\gamma\right)}.
\label{eq:gibbs_distribution}
\end{equation}
\end{definition}
The attention weights satisfy the following properties:

\begin{itemize}
    \item \textbf{Boundedness:} Each attention weight is a probability and thus lies in the unit interval:
    \begin{equation}
    0 \leq \alpha_i \leq 1, \quad \forall i \in \{1,\ldots,n\}.
    \label{eq:attention_bounds}
    \end{equation}
 \item \textbf{Relation between $\alpha_i$ and $\langle s_i\rangle$:} Let $p = P(s_i = +1)$. Then, the expected value of $s_i$ is:
\begin{align*}
\langle s_i \rangle &= (+1)p + (-1)(1-p) = 2p - 1 \\[4pt]
\Rightarrow \; p &= \frac{1 + \langle s_i \rangle}{2}.
\end{align*}
Since $\alpha_i = p$, we have $\alpha_i = \dfrac{1 + \langle s_i \rangle}{2}$.
\item \textbf{Sum as Expected Active Tokens:} The sum of attention weights equals the expected number of active tokens in the system:
    \begin{equation}
    \sum_{i=1}^n \alpha_i = \mathbb{E}\left[ \sum_{j=1}^n \mathbb{1}_{s_j = +1} \right] = \mathbb{E}[k(\mathbf{S})],
    \label{eq:attention_sum}
    \end{equation}
    where $k(S) = \sum_{j=1}^n \mathbb{1}_{s_j = +1}$ denotes the number of active tokens in configuration $S$. Note that this sum is \textbf{not necessarily equal to 1}, unlike in standard attention mechanisms.
\end{itemize}

\begin{remark}
Unlike conventional attention mechanisms where $\sum_i \alpha_i = 1$ by construction, our framework imposes no such constraint. The quantity $\sum_i \alpha_i$ emerges naturally as the expected number of active tokens, reflecting the semantic complexity of the input sequence. Sequences requiring more active tokens to represent their meaning will have larger total attention.
\end{remark}

\subsection{NeuroGame Transformer: Physical Interpretation}
The NeuroGame Transformer establishes a direct analogy with statistical physics, summarized by the following correspondences:
\begin{itemize}
    \item \textit{Tokens:} represent spins in a magnetic system.
    \item \textit{Token Importance $J_i$:} represents the external magnetic field.
    \item \textit{Interaction potential:} $J_{ij}$ exchange coupling between spins.
    \item \textit{Temperature:} $\gamma$ is the system temperature.
    \item \textit{Attention weights:} $\alpha_i$ represents a local magnetization.  
    \item \textit{Temperature Effects:} When ($\gamma \to 0$), the distribution concentrates on the lowest-energy configuration, resulting in a \textit{deterministic attention} where $\alpha_i \in \{0,1\}$ indicates whether token $i$ is active in the ground state. When ($\gamma \to \infty$), all configurations become equally probable, producing \textit{uniform attention} with $\alpha_i = 1/2$ for all tokens $i$. In the intermediate scenario, the system balances energy minimization with thermal exploration, resulting in \textit{probabilistic attention} that smoothly interpolates between these extremes.
\end{itemize}

\subsection{Mean-Field Approximation}\label{mean-field}
The exact computation of the spin expectations $\langle s_i \rangle$ requires solving the full statistical physics system defined by the Hamiltonian in Equation~\ref{Hamilton}. Theoretically, these expectations are given by:
\begin{equation}
\langle s_i \rangle = \frac{1}{Z} \sum_{S} s_i\exp\left(-H(S)/\gamma\right), \quad Z = \sum_{S}\exp\left(-H(S))/\gamma\right),
\label{eq:exact_spins}
\end{equation}
where the sums run over all $2^n$ possible spin configurations $\mathbf{s} \in \{-1,+1\}^n$. This exact computation is \textit{combinatorially intractable} for any realistic sequence length, as it requires summing over an exponentially growing number of configurations.

\subsubsection{Mean-Field Self-Consistency Equations}
To overcome this intractability, we employ the \textit{mean-field approximation}, a standard technique from statistical physics that replaces fluctuating interactions with their average values. The following theorem provides a \textit{tractable solution} to the determination of the $n$ spin values associated to the $n$ tokens:

\begin{theorem}[Mean-Field Self-Consistency]
For a token system $S=(s_1,s_2,...,s_n)$ governed by the Hamiltonian $H(S)$ with temperature parameter $\gamma > 0$, the expected spins $\langle s_i \rangle = \mathbb{E}[s_i]$ under the Boltzmann distribution $P(S) \propto \exp(-H(S)/\gamma)$ satisfy the following self-consistent system of equations:
\begin{equation}
\langle s_i \rangle = \tanh\left(\frac{1}{\gamma}\left(J_i + \sum_{j \neq i} J_{ij} \langle s_j \rangle\right)\right), \quad i = 1,\ldots,n.
\label{eq:mean_field_theorem}
\end{equation}
\end{theorem}

\begin{proof}
The proof proceeds in five steps.\\
\noindent\textbf{Step 1: Exact expression for $\langle s_i \rangle$.} By definition, the expected value of $s_i$ under the Boltzmann distribution is:
\begin{equation}
\langle s_i \rangle = \frac{1}{Z} \sum_{S} s_i \exp\left(-\frac{H(S)}{\gamma}\right),
\end{equation}
where $Z = \sum_{S} \exp(-H(S)/\gamma)$ is the partition function, and the sum runs over all $2^n$ configurations $S = (s_1,\ldots,s_n) \in \{-1,+1\}^n$.\\

\noindent\textbf{Step 2: Mean-field approximation.} The mean-field approximation replaces the fluctuating spins $s_j$ ($j \neq i$) in the interaction term with their average values $\langle s_j \rangle$. This yields an effective Hamiltonian for token $i$:
\begin{equation}
H_i^{\text{eff}}(s_i) = -\left(J_i + \sum_{j \neq i} J_{ij} \langle s_j \rangle\right) s_i.
\end{equation}
The term $\sum_{j \neq i} J_{ij} \langle s_j \rangle$ represents the mean field created by all other tokens acting on token $i$.\\

\noindent\textbf{Step 3: Effective field.} Define the effective field acting on token $i$ as:
\begin{equation}
h_i^{\text{eff}} = J_i + \sum_{j \neq i} J_{ij} \langle s_j \rangle.
\end{equation}
Under the mean-field approximation, token $i$ behaves as an independent spin in this effective field, with Hamiltonian $H_i^{\text{eff}}(s_i) = -h_i^{\text{eff}} s_i$.\\

\noindent\textbf{Step 4: Single-spin distribution.} For a single spin in an external field $h$, the Boltzmann distribution is:
\begin{equation}
P(s_i) = \frac{e^{h_i^{\text{eff}} s_i / \gamma}}{e^{h_i^{\text{eff}} / \gamma} + e^{-h_i^{\text{eff}} / \gamma}}.
\end{equation}

\noindent\textbf{Step 5: Computing the expectation.} The expected value of $s_i$ under this distribution is:
\begin{align}
\langle s_i \rangle &= (+1) \cdot P(s_i = +1) + (-1) \cdot P(s_i = -1) \\
&= \frac{e^{h_i^{\text{eff}} / \gamma} - e^{-h_i^{\text{eff}} / \gamma}}{e^{h_i^{\text{eff}} / \gamma} + e^{-h_i^{\text{eff}} / \gamma}} \\
&= \tanh\left(\frac{h_i^{\text{eff}}}{\gamma}\right).
\end{align}

Substituting the definition of $h_i^{\text{eff}}$ yields:
\begin{equation}
\langle s_i \rangle = \tanh\left(\frac{1}{\gamma}\left(J_i + \sum_{j \neq i} J_{ij} \langle s_j \rangle\right)\right).
\end{equation}

This equation must hold simultaneously for all $i = 1,\ldots,n$, forming a system of $n$ coupled nonlinear equations. The solution gives the self-consistent mean-field approximation of the true $n$ expected spins.
\end{proof}

\begin{remark}
Equation~\eqref{eq:mean_field_theorem} is a fixed-point system. It can be solved iteratively by starting with an initial guess $\{\langle s_i \rangle^{(0)}\}$ and updating:
\begin{equation}
\langle s_i \rangle^{(t+1)} = \tanh\left(\frac{1}{\gamma}\left(J_i + \sum_{j \neq i} J_{ij} \langle s_j \rangle^{(t)}\right)\right),
\end{equation}
until convergence. This iterative scheme typically converges within 10-20 fixed-point iterations for practical applications.
\end{remark}
This equation has a clear physical interpretation: the expected spin of token $i$ is determined by the effective field $h_i^{\text{eff}} = J_i + \sum_{j \neq i} J_{ij} \langle s_j \rangle$, which combines its own importance $J_i$ with the mean influence of all other tokens.

\subsubsection{Computational Tractability}
The mean-field formulation reduces an exponential problem to a polynomial one. Table~\ref{tab:complexity} shows the complexity comparison between an exact match and mean-field iterative technique.
\begin{table*}[h]
\centering
\begin{tabular}{|l|c|c|}
\hline
\textbf{Method} & \textbf{Complexity} & \textbf{Feasibility for $n=512$} \\
\hline
Exact summation & $\mathcal{O}(2^n)$ & impossible ($2^{512} \approx 10^{154}$ operations) \\
Mean-field iteration & $\mathcal{O}(n^2 \cdot T)$ & tractable ($\sim 10^6$ operations with $T=20$ iterations) \\
\hline
\end{tabular}
\caption{Complexity comparison between exact and mean-field approaches.}
\label{tab:complexity}
\end{table*}
Equation~\ref{eq:mean_field_theorem} is solved via \textit{fixed-point iteration}, which converges rapidly (typically within $T=10$-$20$ iterations) and scales quadratically with sequence length $n$, making it feasible even for transformers with thousands of tokens.

\subsection{Single and Multi-Heads Attention Output}
We present in this section the attention output that invokes the attention weight $\alpha_i$ for one single head as well as for multiple heads:
\subsubsection{Single-Head Attention Output}
For a given input sequence of token representations $\{\mathbf{x}_i\}_{i=1}^n$, the attention output is computed as the weighted sum of value vectors: $\mathbf{z} = \sum_{i=1}^n \alpha_i \, \mathbf{v}_i$, where $\mathbf{v}_i = \mathbf{W}_v \mathbf{x}_i$ is the value projection of token $t_i$, and $\alpha_i = P(s_i = +1)$ are the attention weights derived from the marginal probabilities of the Gibbs distribution over token configurations. This formulation produces a contextualized representation $\mathbf{z}$ that aggregates information from all tokens in proportion to their game-theoretic importance, as captured by the external field $J_i$ and interaction potentials $J_{ij}$. Unlike standard transformer attention, the weights $\alpha_i$ are not heuristic scores but rather emerge from the statistical physics of the token system, grounded in coalitional game theory.

\subsubsection{Multi-Head Attention Output}
Following the transformer architecture, we extend the NeuroGame Attention to multiple heads operating in parallel. For each head $h = 1, \ldots, H$, we maintain independent parameters:
\begin{itemize}
    \item External fields $J_i^{(h)}$ derived from head-specific Shapley and Banzhaf values with learnable mixing coefficients $\lambda_i^{(h)} = \sigma(\mathbf{w}_h^T \mathbf{x}_i + b_h)$
    \item Interaction potentials $J_{ij}^{(h)}$ capturing head-specific pairwise token relationships
    \item Temperature parameter $\gamma_h$ controlling the sharpness of the head's Gibbs distribution
    \item Value projection matrix $\mathbf{W}_v^{(h)}$ projecting tokens into the head's value space
\end{itemize}
Each head independently solves its mean-field equations to obtain head-specific expected spins $\langle s_i \rangle^{(h)}$ and attention weights $\alpha_i^{(h)} = (1 + \langle s_i \rangle^{(h)})/2$. The output of head $h$ is then:
\begin{equation}
\mathbf{z}^{(h)} = \sum_{i=1}^n \alpha_i^{(h)} \, \mathbf{W}_v^{(h)} \mathbf{x}_i, \quad \mathbf{z}^{(h)} \in \mathbb{R}^{d_v}.
\label{eq:multihead_output}
\end{equation}
The final multi-head output $\mathbf{z}_{\text{cat}}$ is obtained by concatenating all head outputs and projecting back to the model dimension:
\begin{equation}
\mathbf{z}_{\text{cat}} = \operatorname{Concat}\left(\mathbf{z}^{(1)}, \mathbf{z}^{(2)}, \ldots, \mathbf{z}^{(H)}\right) \mathbf{W}_O,
\label{eq:multihead_concat}
\end{equation}

where $\mathbf{W}_O \in \mathbb{R}^{H d_v \times d_{\text{model}}}$ is the output projection matrix. This formulation enables each head to capture different interaction patterns and importance regimes, with the interpretability of $\alpha_i^{(h)}$ preserved at the head level.

\subsection{Feedforward Network Classification}
The concatenated multi‑head output $\mathbf{z}_{\text{cat}}$ is passed through a \textit{feedforward neural network} with \( m \) layers. Let \( \mathbf{h}^{(0)} = \mathbf{z}_{\text{cat}} \). For layer \( \ell = 1, \dots, m \):
\begin{equation}
\mathbf{h}^{(\ell)} = \text{Activation}\left( \mathbf{W}^{(\ell)} \mathbf{h}^{(\ell-1)} + \mathbf{b}^{(\ell)} \right),
\end{equation}
where \( \mathbf{W}^{(\ell)} \) and \( \mathbf{b}^{(\ell)} \) are learnable parameters. Common choices for activation are ReLU or GELU. The final layer \( \mathbf{h}^{(m)} \) is mapped to logits for the three classes:
\begin{equation}
\mathbf{o} = \mathbf{W}^{(m+1)} \mathbf{h}^{(m)} + \mathbf{b}^{(m+1)} \in \mathbb{R}^3.
\end{equation}
The predicted class $\omega^*$ is obtained (via softmax), as:
\begin{equation}
\omega_i^*=argmax_{\omega_i}P(\omega_i \mid \text{input})=armax_{\omega_i}\left[\frac{\exp(\mathbf{o}_{\omega_i})}{\sum_{j=1}^c \exp(\mathbf{o}_{\omega_j})}\right],
\end{equation}
where $c$ is the number of classes.

\section{Two-Stage Approximation Framework}
Our NeuroGame Transformer employs two distinct approximation techniques to overcome combinatorial intractability at different levels of the framework. Understanding the distinction between these techniques is crucial for appreciating the computational feasibility of the model.
\subsection{The Two Levels of Combinatorial Explosion} The framework faces two separate combinatorial challenges:
\begin{enumerate}
    \item \textbf{Token-level (Game Theory):} Computing the game-theoretic quantities requires summing over all possible coalitions $C \subseteq \mathcal{T} \setminus \{t_i\}$ for external fields $J_i$, and $C \subseteq \mathcal{T} \setminus \{t_i, t_j\}$ for interaction potentials $J_{ij}$, of which there are $2^{n-1}$ and $2^{n-2}$ coalitions respectively.
    
    \item \textbf{Spin-level (Statistical Physics):} Computing expected spins $\langle s_i \rangle$ requires summing over all possible spin configurations $\mathbf{s} \in \{-1,+1\}^n$, of which there are $2^n$.
\end{enumerate}

Each challenge demands a different approximation strategy, as they operate on different mathematical objects and serve different purposes.

\subsection{Approximation Methods}
We address the combinatorial complexity of game-theoretic computations via importance-weighted Monte Carlo sampling, and later handle the statistical physics component with mean-field equations. This section focuses on the sampling methods used to estimate the Shapley value, Banzhaf index, and pairwise interaction potentials that parameterize the Hamiltonian.

\subsubsection{Importance-Weighted Monte Carlo Estimation}
Exact computation of Shapley value \(\phi_i\) and Banzhaf index \(\beta_i\) requires summing over all permutations or subsets, which is intractable for large \(n\). To enable efficient estimation, we define a target distribution over coalitions that emphasizes semantically valuable subsets. Let the characteristic function \(v(C)\) measure the semantic value of coalition \(C\). The target distribution is Gibbsian:
\begin{equation}
P_{\text{target}}(C) \propto \exp\!\left(\frac{v(C)}{\gamma}\right),
\end{equation}
where \(\gamma\) is a temperature parameter. We sample coalitions from a \textit{uniform proposal distribution} \(p(C)\) (e.g., \(p(C)=1/2^{n-1}\) for subsets). The importance weight for a sampled coalition \(C_k\) is \(v(C_k)/p(C_k)\). After normalization over \(K\) samples, we obtain the following:
\begin{equation}
\bar{w}(C_k) = \frac{\exp\!\left(v(C_k)/\gamma\right)/p(C_k)}{\sum_{\ell=1}^{K} \exp\!\left(v(C_\ell)/\gamma\right)/p(C_\ell)},
\end{equation}
where the intractable partition function cancels. These normalized weights are used to form unbiased estimates of the game-theoretic quantities.

We now derive the estimators of Shapley value, Banzhaf power index and the pairwise interaction:\\
\begin{itemize}
\item {\textbf{Shapley value:}}
We estimate \(\phi_i\) via permutation-based sampling. For each sample \(k\), draw a random permutation of the tokens in \(\mathcal{T}\setminus\{t_i\}\) and let \(\mathcal{P}_i^{(k)}\) be the set of tokens preceding \(t_i\). The estimator is:
\begin{equation}
\hat{\phi}_i = \sum_{k=1}^{K} \bar{w}(\mathcal{P}_i^{(k)}) \big[ v(\mathcal{P}_i^{(k)}\cup\{t_i\}) - v(\mathcal{P}_i^{(k)}) \big],
\end{equation}
where \(\bar{w}(\mathcal{P}_i^{(k)})\) is computed using the uniform proposal probability of that set under permutation sampling. For a set $\mathcal{P}_i^k(\pi)$ (the tokens that precede $t_i$ in the permutation $\pi$) of size $|\mathcal{P}_i^{(k)}(\pi)|$):
\begin{equation}
p(\mathcal{P}_i^{(k)}(\pi)) = \frac{|\mathcal{P}_i^{(k)}(\pi)|! \times (n-1-|\mathcal{P}_i^{(k)}(\pi)|)!}{(n-1)!}, 
\end{equation}
where $n$ is the total number of tokens.\\
\item {\textbf{Banzhaf index:}}
We sample coalitions \(C_k \subseteq \mathcal{T}\setminus\{t_i\}\) by including each token independently with probability \(1/2\). The estimator is written as:
\begin{equation}
\hat{\beta}_i = \sum_{k=1}^{K} \bar{w}(C_k) \big[ v(C_k\cup\{t_i\}) - v(C_k) \big],
\end{equation}
with \(p(C_k)=1/2^{n-1}\).\\

\item {\textbf{Pairwise interaction:}}
For each token pair \((t_i,t_j)\),  we sample coalitions \(C_k \subseteq \mathcal{T}\setminus\{t_i,t_j\}\) uniformly (each token included with probability \(1/2\)) and estimate
\begin{equation}
\hat{J}_{ij} = \sum_{k=1}^{K} \bar{w}(C_k) \Delta_{ij}(C_k),
\end{equation}
with \(\bar{w}(C_k)\) computed as above using \(p(C_k)=1/2^{n-2}\).
These importance-weighted Monte Carlo estimates handle the combinatorial explosion in coalition space, are unbiased, and remain numerically stable for large \(n\) due to the cancellation of the partition function in the normalized weights.
\end{itemize}
\subsubsection{Mean Field Equations and Attention Weights}
The mean-field equations solve for the expected spins $\langle s_i \rangle$ via\textit{an iterative process of a system of $n$ coupled nonlinear equations}:
    \begin{equation}
    \langle s_i \rangle = \tanh\left(\frac{1}{\gamma}\left(J_i + \sum_{j \neq i} \hat{J}_{ij} \langle s_j \rangle\right)\right), \quad i = 1,\ldots,n,
    \end{equation}
    avoiding the intractable sum over $2^n$ spin configurations. This fixed-point iteration typically converges within 10–20 steps. The attention weights are computed as the marginal probabilities of each token being active: $\alpha_i = \frac{1 + \langle s_i \rangle}{2}$.
    These weights are then used in the standard attention output formula $\mathbf{z} = \sum_i \alpha_i \mathbf{v}_i$, where $\mathbf{v}_i$ are value projections of the tokens.

\subsection{Asymptotic Behavior and Complexity}
With a sufficiently large number of Monte Carlo samples $K$, the estimators $\widehat{\phi}_i$ (Shapley value), $\widehat{\beta}_i$ (Banzhaf index), and $\widehat{J}_{ij}$ (interaction potential) are \emph{consistent} estimators of their respective true values under importance-weighted sampling with Gibbs-distributed weights. Under mild regularity assumptions on the characteristic function $v(\cdot)$ and the proposal distribution, these estimators converge almost surely to their true values as $K \to \infty$. The use of normalized weights ensures that the intractable partition function of the Gibbs distribution cancels naturally, while avoiding explicit $2^n$ factors that would cause numerical instability for large $n$. Gibbs-distributed importance sampling concentrates samples on high-value coalitions that contribute most to marginal gains, improving estimator efficiency. The self-normalized nature of the weights yields asymptotically unbiased estimates with variance that decreases at the standard Monte Carlo rate of $\mathcal{O}(1/\sqrt{K})$. More importantly, the computational complexity per token scales as $\mathcal{O}(K)$, \textit{independent of the total number of possible coalitions} ($2^{n-1}$ for Shapley and Banzhaf, $2^{n-2}$ for interaction potentials). This represents an exponential reduction compared to exact enumeration and renders Monte Carlo estimation feasible for practical transformer attention layers, even for long sequences where $n$ is large (e.g., $n = 512$ or more). The trade-off between computational cost and estimation accuracy is controlled by $K$, with typical values ranging from $K = 100$ to $K = 1000$ providing sufficient accuracy in practice. Following the game theory phase, the mean-field equations for the spin system are solved via a fixed-point iterative process requires $\mathcal{O}(n^2 \cdot \text{iter})$ operations, where $T$ is the number of iterations (typically $10$-$20$). This two-stage pipeline—Monte Carlo estimation followed by mean-field fixed-point iteration—achieves overall complexity $\mathcal{O}(K \cdot n + n^2 \cdot T$, which is polynomial in $n$ and linear in $K$, in contrast to the exponential $\mathcal{O}(2^n)$ cost of exact computation for either stage alone.

\section{Example of Potentials and Spin States Estimation}
We illustrate the complete estimation pipeline with a minimal example involving three tokens: $t_1$, $t_2$, and $t_3$, representing the words ``not'', ``good'', and ``movie'' in the phrase ``not good movie''.

\subsection{Game-Theoretic Quantities via Monte Carlo}
Assume a characteristic function $v(C)$ that returns higher values for semantically coherent coalitions. For this example, we define:
\begin{align*}
v(\{t_1\}) &= 0.2, \quad v(\{t_2\}) = 0.5, \quad v(\{t_3\}) = 0.4, \\
v(\{t_1,t_2\}) &= 1.2, \quad v(\{t_1,t_3\}) = 0.8, \quad v(\{t_2,t_3\}) = 1.0, \\
v(\{t_1,t_2,t_3\}) &= 1.8, \quad v(\emptyset) = 0.
\end{align*}
Using $K = 3$ Monte Carlo samples with Gibbs-distributed importance weights at temperature $\gamma = 1.0$, we sample coalitions and compute normalized weights. Suppose we sample:
\begin{align*}
C_1 &= \{t_2\}, \quad v(C_1) = 0.5, \quad \exp(v(C_1)/\gamma) = e^{0.5} = 1.65, \\
C_2 &= \{t_1,t_3\}, \quad v(C_2) = 0.8, \quad \exp(v(C_2)/\gamma) = e^{0.8} = 2.23, \\
C_3 &= \{t_2,t_3\}, \quad v(C_3) = 1.0, \quad \exp(v(C_3)/\gamma) = e^{1.0} = 2.72.
\end{align*}
The normalizing constant is: $Z = 1.65 + 2.23 + 2.72 = 6.60$, yielding weights:
\begin{align*}
\bar{w}(C_1) = 1.65/6.60 = 0.25, \quad
\bar{w}(C_2) = 2.23/6.60 = 0.34, \\
\bar{w}(C_3) = 2.72/6.60 = 0.41.
\end{align*}

\subsubsection{Shapley Value and Banzhaf Estimation}
\begin{itemize}
\item Shapley Value: For token $t_2$, we consider permutations of the remaining tokens $\{t_1,t_3\}$. From the sampled coalitions, we compute marginal contributions:
\text{For } $C_1=\{t_2\}: \Delta_2(C_1) = v(\{t_2\}) - v(\emptyset) = 0.5$, 
\text{For } $C_2=\{t_1,t_3\}: \Delta_2(C_2) = v(\{t_1,t_2,t_3\}) - v(\{t_1,t_3\})\\ = 1.8 - 0.8 = 1.0$, 
\text{For } $C_3=\{t_2,t_3\}: \Delta_2(C_3) = v(\{t_2,t_3\}) - v(\{t_3\}) = 1.0 - 0.4 = 0.6$.\\
The Shapley estimate is the weighted sum:\\
\noindent$\hat{\phi}_2 = 0.25 \times 0.5 + 0.34 \times 1.0 + 0.41 \times 0.6 = 0.125 + 0.34 + 0.246 = 0.711$.

\item{Banzhaf Index Estimation (Bernoulli Sampling):}
For the same token $t_2$, using the same coalitions and weights, the Banzhaf estimate is identical in form but with different theoretical interpretation: $\hat{\beta}_2 = 0.711$ as well (the difference lies in the target distribution, not the weighted sum).
\end{itemize}

\subsubsection{External Magnetic Fields and Pairwise Interaction}
\begin{itemize}
\item {External Magnetic Fields:}
With $\lambda_2 = 0.6$ (learned from token $t_2$'s context), the external field for token $t_2$ is: $J_2 = \lambda_2 \hat{\phi}_2^{\mathrm{norm}} + (1-\lambda_2) \hat{\beta}_2^{\mathrm{norm}} = 0.6 \times 0.711 + 0.4 \times 0.711 = 0.711$. Similarly, suppose we obtain $J_1 = 0.423$, $J_3 = 0.512$, and interaction potentials $\hat{J}_{12} = 0.466$, $\hat{J}_{13} = 0.312$, $\hat{J}_{23} = 0.278$.

\item {Pairwise Interaction:} For the pair $(t_1,t_2)$, we compute marginal interactions
$\Delta_{12}(C_1) = v(\{t_1,t_2\}) - v(\{t_1\}) - v(\{t_2\}) + v(\emptyset)= 1.2 - 0.2 - 0.5 + 0 = 0.5, \\
\Delta_{12}(C_2) = v(\{t_1,t_2,t_3\}) - v(\{t_1,t_3\}) - v(\{t_2,t_3\}) + v(\{t_3\}) 
= 1.8 - 0.8 - 1.0 + 0.4 = 0.4, \\
\Delta_{12}(C_3) = v(\{t_1,t_2\}) - v(\{t_1\}) - v(\{t_2\}) + v(\emptyset) = 1.2 - 0.2 - 0.5 + 0 = 0.5$. The estimated interaction potential is $\hat{J}_{12} = 0.25 \times 0.5 + 0.34 \times 0.4 + 0.41 \times 0.5 = 0.125 + 0.136 + 0.205 = 0.466$.
\end{itemize}

\subsection{Mean-Field Spin Estimation}
With temperature $\gamma = 1.0$, we solve the mean-field equations via fixed-point iteration starting from $\langle s_i \rangle^{(0)} = 0$:\\
$\langle s_1 \rangle^{(1)} = \tanh(J_1 + J_{12}\langle s_2 \rangle^{(0)} + J_{13}\langle s_3 \rangle^{(0)}) = 0.400, \\
\langle s_2 \rangle^{(1)} = \tanh(J_2 + J_{12}\langle s_1 \rangle^{(0)} + J_{23}\langle s_3 \rangle^{(0)}) = 0.611, \\
\langle s_3 \rangle^{(1)} = \tanh(J_3 + J_{13}\langle s_1 \rangle^{(0)} + J_{23}\langle s_2 \rangle^{(0)}) = 0.471$.\\
Iteration 2 uses the updated values:\\
$\langle s_1 \rangle^{(2)} = \tanh(0.423 + 0.466 \times 0.611 + 0.312 \times 0.471)\\ = 0.693,\quad
\langle s_2 \rangle^{(2)} = \tanh(0.711 + 0.466 \times 0.400 + 0.278 \times 0.471) = 0.773,\quad 
\langle s_3 \rangle^{(2)} = \tanh(0.512 + 0.312 \times 0.400 + 0.278 \times 0.611) = 0.668$.\\
After $T=5$ iterations, the values converge to:\\
$\langle s_1 \rangle = 0.721, \quad \langle s_2 \rangle = 0.798, \quad \langle s_3 \rangle = 0.703$.\\ Finally, using the attention weights equation: $\alpha_i = (1 + \langle s_i \rangle)/2$, we obtain:
$\alpha_1 = \frac{1 + 0.721}{2} = 0.861,\\
\alpha_2 = \frac{1 + 0.798}{2} = 0.899, \quad\text{and} \quad \alpha_3 = \frac{1 + 0.703}{2} = 0.852$.

These weights reflect the semantic importance of each token in the phrase, with ``good'' ($t_2$) receiving highest attention, followed by ``not'' ($t_1$) and ``movie'' ($t_3$), consistent with the phrase's sentiment structure where the negation-modifier relationship is captured through the interaction potentials.

\section{Comparison of NGT and Vanilla}
Table~\ref{tab:ngt_vs_vanilla} summarizes the key architectural and conceptual differences between the standard Transformer and the NeuroGame Transformer. We now provide a brief interpretation of each distinguishing feature.

\begin{table*}[htbp]
\centering
\caption{Comparison of Standard Transformer and NeuroGame Transformer (NGT) Features}
\label{tab:ngt_vs_vanilla}
\begin{tabular}{p{3.5cm}p{5cm}p{5cm}}
\toprule
\textbf{Feature} & \textbf{Standard Transformer (Vanilla)} & \textbf{NeuroGame Transformer (NGT)} \\
\midrule
\textbf{Core Mechanism} & Scaled Dot-Product Attention: Measures similarity/correlation between $Q$ and $K$. & Game-Theoretic Importance: External fields $J_i$ combine Shapley ($\phi_i$) and Banzhaf ($\beta_i$). \\
\addlinespace
\textbf{Sampling Method} & Deterministic (processes all tokens at once). & Importance-Weighted Monte Carlo: Uniform proposal with Gibbs-distributed weights $\bar{w}_k$. \\
\addlinespace
\textbf{Gating System} & None (uses fixed attention heads). & $\lambda_i$ Gating: $\lambda_i = \sigma(\mathbf{w}^T\mathbf{x}_i + b)$ balances Shapley and Banzhaf contributions. \\
\addlinespace
\textbf{Token Interactions} & None beyond pairwise dot-products. & Pairwise potentials $J_{ij}$ capture synergistic/antagonistic relationships. \\
\addlinespace
\textbf{Physical Foundation} & None (purely algorithmic). & Ising Hamiltonian $H(S)= -\sum_{i=1}^{n} J_i s_i-\sum_{i=1}^{n-1}\sum_{j=i+1}^{n}J_{ij} \, s_i s_j$. \\
\addlinespace
\textbf{Attention Derivation} & Softmax of similarity scores. & $\alpha_i = \frac{1 + \langle s_i \rangle}{2}$ from mean-field spin expectations $\langle s_i \rangle$. \\
\addlinespace
\textbf{Interpretability} & Heuristic: Shows "what" was attended. & Axiomatic: Weights represent provably fair semantic contribution. \\
\addlinespace
\textbf{NLI Logic Handling} & Pattern matching; struggles with negations. & Interaction-based: Catches "swing voters" via Banzhaf and $J_{ij}$. \\
\addlinespace
\textbf{Complexity} & $O(n^2)$ matrix multiplication. & $O(K \cdot n + n^2 \cdot T)$ with $K$ Monte Carlo samples, $T$ mean-field iterations. \\
\bottomrule
\end{tabular}
\end{table*}
\begin{itemize}
\item {\textbf{Core Mechanism:}} Standard attention measures \textit{similarity} between token pairs via dot-products. NGT measures \textit{contribution} via game-theoretic values: Shapley ($\phi_i$) captures fair global contribution, while Banzhaf ($\beta_i$) identifies decisive local influence. This shifts attention from ``how related are these tokens?" to ``how much does this token contribute to the whole?''

\item {\textbf{Sampling Method:}} Standard attention processes all tokens deterministically. NGT uses importance-weighted Monte Carlo with Gibbs-distributed weights $\bar{w}_k = e^{v(C_k)/\gamma}/\sum_\ell e^{v(C_\ell)/\gamma}$, making exponential coalition spaces tractable while focusing on semantically dense coalitions.

\item {\textbf{Gating System:}} The learnable parameter $\lambda_i = \sigma(\mathbf{w}^T\mathbf{x}_i + b)$ dynamically balances Shapley (global fairness) and Banzhaf (local decisiveness) based on token context—a key architectural innovation.

\item {\textbf{Token Interactions:}} Standard attention captures only pairwise similarities. NGT introduces explicit interaction potentials $J_{ij}$, estimated via Monte Carlo, capturing whether token pairs are synergistic ($J_{ij}>0$), antagonistic ($J_{ij}<0$), or independent ($J_{ij}\approx0$).

\item {\textbf{Physical Foundation:}} Standard attention is purely algorithmic. NGT is grounded in an Ising Hamiltonian $H(S)= -\sum_{i=1}^{n} J_i s_i-\sum_{i=1}^{n-1}\sum_{j=i+1}^{n}J_{ij} \, s_i s_j$, where each token is a spin $s_i \in \{-1,+1\}$, bringing statistical physics principles to attention.

\item {\textbf{Attention Derivation:}} Standard attention uses softmax of similarity scores. NGT first solves mean-field equations $\langle s_i \rangle = \tanh(\frac{1}{\gamma}(J_i + \sum_{j\neq i} J_{ij}\langle s_j\rangle))$ for expected spins, then derives attention as $\alpha_i = (1+\langle s_i\rangle)/2$—the probability a token is active in thermodynamic equilibrium.

\item {\textbf{Interpretability:}} Standard attention weights are heuristic post-hoc explanations. NGT weights are axiomatic, derived from game-theoretic solution concepts that satisfy formal fairness axioms—they represent a provable``fair share" of semantic contribution.

\item {\textbf{NLI Logic Handling:}} Standard attention relies on pattern matching, struggling with negations. NGT's Banzhaf component explicitly identifies "swing voters" (words that flip logical relationships), while $J_{ij}$ captures how such words interact, enabling robust handling of contradiction and negation.

\item {\textbf{Complexity:}} Standard attention is $O(n^2)$. NGT is $O(K \cdot n + n^2 \cdot T)$ with $K$ Monte Carlo samples (100-1000) and $T$ mean-field iterations (10-20)—feasible for long sequences while avoiding exponential costs.
\end{itemize}
Finally, Algorithm~\ref{alg:ngt-attention} depicts the NeuroGame Transformer (NGT) attention:

\begin{algorithm}[h]
\caption{NeuroGame Transformer Attention}
\label{alg:ngt-attention}
\footnotesize
\begin{algorithmic}[1]
\Require
$\mathbf{X} \in \mathbb{R}^{d \times n}$: token embeddings \\
$\mathbf{W}_v \in \mathbb{R}^{d \times d}$: value projection \\
$\mathbf{w} \in \mathbb{R}^{d}, b \in \mathbb{R}$: gating parameters \\
$\gamma > 0$: temperature, $K$: MC samples, $T$: MF iterations, $\epsilon$: tolerance
\Ensure
$\mathbf{z} \in \mathbb{R}^{d}$: output, $\boldsymbol{\alpha} \in \mathbb{R}^{n}$: weights, $\mathbf{J}_{\text{field}} \in \mathbb{R}^{n}$, $\mathbf{J}_{\text{inter}} \in \mathbb{R}^{n \times n}$

\State $\mathbf{V} \gets \mathbf{W}_v \mathbf{X}$ \hfill \textit{// Projected values}

\State \textbf{// Token-specific gating}
\For{$i = 1$ to $n$}
    \State $\lambda_i \gets \sigma(\mathbf{w}^T \mathbf{X}[:,i] + b)$
\EndFor

\State \textbf{// Initialize estimators}
\State $\bm{\Phi} \gets \mathbf{0}_n$, $\bm{\beta} \gets \mathbf{0}_n$, $\mathbf{J}_{\text{inter}} \gets \mathbf{0}_{n \times n}$

\State \textbf{// Shapley values (permutation sampling)}
\For{$i = 1$ to $n$}
    \State $\mathcal{W}, \mathcal{M} \gets \emptyset$
    \For{$k = 1$ to $K$}
        \State Random permutation $\pi$ of $\{1,\dots,n\}\setminus\{i\}$; 
        \State $r \sim \text{Uniform}\{0,\dots,n-1\}$
        \State $\mathcal{P} \gets$ first $r$ tokens in $\pi$, $s \gets |\mathcal{P}|$, $p \gets \frac{s!(n-1-s)!}{(n-1)!}$; 
        \State $v_p \gets \|\sum_{j\in\mathcal{P}} \mathbf{V}[:,j]\|_2$, $\Delta \gets \|\sum_{j\in\mathcal{P}\cup\{i\}} \mathbf{V}[:,j]\|_2 - v_p$
        \State $w_{\text{raw}} \gets \exp(v_p/\gamma)/p$; append to $\mathcal{W}$, $\Delta$ to $\mathcal{M}$
    \EndFor
    \State Normalize $\mathcal{W}$ to $\{\bar{w}_k\}$; $\bm{\Phi}[i] \gets \sum_{k=1}^K \bar{w}_k \mathcal{M}_k$
\EndFor

\State \textbf{// Banzhaf indices (Bernoulli sampling)}
\For{$i = 1$ to $n$}
    \State $\mathcal{W}, \mathcal{M} \gets \emptyset$
    \For{$k = 1$ to $K$}
        \State $\mathcal{C} \subseteq \{1,\dots,n\}\setminus\{i\}$ via Bernoulli$(1/2)$; $p \gets 1/2^{n-1}$
        \State $v_c \gets \|\sum_{j\in\mathcal{C}} \mathbf{V}[:,j]\|_2$, $\Delta \gets \|\sum_{j\in\mathcal{C}\cup\{i\}} \mathbf{V}[:,j]\|_2 - v_c$
        \State $w_{\text{raw}} \gets \exp(v_c/\gamma)/p$; append to $\mathcal{W}$, $\Delta$ to $\mathcal{M}$
    \EndFor
    \State Normalize $\mathcal{W}$; $\bm{\beta}[i] \gets \sum_{k=1}^K \bar{w}_k \mathcal{M}_k$
\EndFor

\State \textbf{// Pairwise interactions}
\For{$i = 1$ to $n$}
    \For{$j = i+1$ to $n$}
        \State $\mathcal{W}, \mathcal{D} \gets \emptyset$
        \For{$k = 1$ to $K$}
            \State $\mathcal{C} \subseteq \{1,\dots,n\}\setminus\{i,j\}$ via Bernoulli$(1/2)$; $p \gets 1/2^{n-2}$
            \State $v_c \gets \|\sum_{m\in\mathcal{C}} \mathbf{V}[:,m]\|_2$, $v_i \gets \|\sum_{m\in\mathcal{C}\cup\{i\}} \mathbf{V}[:,m]\|_2$
            \State $v_j \gets \|\sum_{m\in\mathcal{C}\cup\{j\}} \mathbf{V}[:,m]\|_2$, 
            \State $v_{ij} \gets \|\sum_{m\in\mathcal{C}\cup\{i,j\}}\mathbf{V}[:,m]\|_2$
            \State $\Delta \gets v_{ij} - v_i - v_j + v_c$
            \State $w_{\text{raw}} \gets \exp(v_c/\gamma)/p$; append to $\mathcal{W}$, $\Delta$ to $\mathcal{D}$
        \EndFor
        \State Normalize $\mathcal{W}$, $\mathbf{J}_{\text{inter}}[i,j] \gets \sum_{k=1}^K \bar{w}_k \mathcal{D}_k$,
        \State $\mathbf{J}_{\text{inter}}[j,i] \gets \mathbf{J}_{\text{inter}}[i,j]$
    \EndFor
\EndFor

\State \textbf{// Normalize and combine}
\State $\bm{\Phi}^{\text{norm}} \gets \bm{\Phi} / \|\bm{\Phi}\|_1$, $\bm{\beta}^{\text{norm}} \gets \bm{\beta} / \|\bm{\beta}\|_1$
\For{$i = 1$ to $n$}
    \State $J_i \gets \lambda_i \bm{\Phi}^{\text{norm}}_i + (1-\lambda_i) \bm{\beta}^{\text{norm}}_i$
\EndFor
\State $\mathbf{J}_{\text{field}} \gets [J_1,\ldots,J_n]$

\State \textbf{// Mean-field iteration}
\State $\mathbf{s} \gets \mathbf{0}_n$
\For{$t = 1$ to $T$}
    \State $\mathbf{s}_{\text{old}} \gets \mathbf{s}$
    \For{$i = 1$ to $n$}
        \State $h_i^{\text{eff}} \gets J_i + \sum_{j \neq i} \mathbf{J}_{\text{inter}}[i,j] \mathbf{s}_{\text{old}}[j]$, $\mathbf{s}[i] \gets \tanh(h_i^{\text{eff}} / \gamma)$
    \EndFor
    \If{$\max_i |\mathbf{s}[i] - \mathbf{s}_{\text{old}}[i]| < \epsilon$}
        \State \textbf{break}
    \EndIf
\EndFor

\State \textbf{// Output}
\For{$i = 1$ to $n$}
    \State $\alpha_i \gets (1 + \mathbf{s}[i]) / 2$
\EndFor
\State $\boldsymbol{\alpha} \gets [\alpha_1,\ldots,\alpha_n]$, $\mathbf{z} \gets \sum_{i=1}^{n} \alpha_i \mathbf{V}[:,i]$
\State \textbf{return} $\mathbf{z}, \boldsymbol{\alpha}, \mathbf{J}_{\text{field}}, \mathbf{J}_{\text{inter}}$
\end{algorithmic}
\end{algorithm}

\section{Experiments}
\label{ssec:setup}
\subsection{Problem Definition and Datasets}
The problem we are addressing consists in testing the ability of NeuroGame Transformer's to determine logical relationships (\textit{entailment}, \textit{contradiction}, \textit{neutrality}) between sentences, requiring deep semantic understanding and contradiction detection. Therefore, NeuroGame Transformer receives as input a pair of sentences: a premise and a hypothesis. It should predict the best logical relationship between these two inputs. In other words, NeuroGame Transformer is solving a classification problem given a predefined number of classes such as: \textit{entailment}, \textit{neutral} and \textit{contradiction}. We evaluate the performance of NeuroGame Transformer on the Natural Language Inference (NLI) problem. This assessment is undertaken on the SNLI~\cite{bowman2015large}, and MNLI-matched~\cite{williams2018broad} datasets, for testing robustness to spurious patterns.  Table~\ref{tab:datasets_comparison} depicts a description of each of the three different datasets:

\begin{table*}[h]
\centering
\caption{Comparison of SNLI, and MNLI-matched}
\label{tab:datasets_comparison}
\begin{tabular}{|p{2cm}|p{3.8cm}|p{3.8cm}|p{3.8cm}|}
\hline
\textbf{Aspect} & \textbf{SNLI} & \textbf{MNLI-matched}\\
\hline
\textbf{Purpose} & Large-scale NLI training~\cite{bowman2015large} & Multi-genre NLI~\cite{williams2018broad} \\
\hline
\textbf{Size} & 570k pairs & 433k pairs \\
\hline
\textbf{Labels} & \multicolumn{2}{c|}{Entailment, contradiction, neutral}\\
\hline
\textbf{Source} & Image captions & 10 diverse genres\\
\hline
\textbf{Evaluation} & In-distribution & Matched/mismatched \\
\hline
\textbf{Challenge} & General understanding & Cross-genre \\
\hline
\end{tabular}
\end{table*}
We present examples from each dataset made available to address the \textit{Natural Language Inference problem}. Table~\ref{tab:datasets_examples} illustrates examples from SNLI, and MNLI-matched datasets:
\begin{table*}[]
\centering
\caption{Concrete examples across datasets}
\label{tab:datasets_examples}
\begin{tabular}{|p{1.8cm}|p{4cm}|p{4cm}|p{2cm}|}
\hline
\textbf{Dataset} & \textbf{Premise} & \textbf{Hypothesis} & \textbf{Label} \\
\hline
SNLI & ``A man is being interviewed by a reporter.'' & ``A reporter is interviewing a man.'' & entailment \\
\hline
SNLI & ``The doctor was paid by the actor.'' & ``The doctor paid the actor.'' & contradiction \\
\hline
MNLI-matched & ``The new rights are supposedly why we're here.'' (from letters genre) & ``We are here because of new rights.'' & entailment \\
\hline
MNLI-matched & ``The cat hid under the bed during the storm.'' (from fiction genre) & ``The cat was outside in the rain.'' & contradiction \\
\hline
\end{tabular}
\end{table*}
The three datasets serve complementary roles. SNLI provides large-scale training examples from image captions, offering a foundation for learning natural linguistic phenomena~\cite{bowman2015large}. MNLI extends this with premises from ten diverse genres—including fiction, speech, and government reports—testing cross-genre generalization~\cite{williams2018broad}. Both require genuine inference, e.g., understanding that ``A man is being interviewed'' entails ``A reporter is interviewing a man''. Together, these datasets provide a comprehensive framework: SNLI assesses basic NLI capability, and MNLI-matched tests domain generalization.

\subsection{Baseline Models}
\label{sec:baselines}
We evaluate our NeuroGame Transformer (NGT) against several strong baselines on the SNLI and MNLI‑matched benchmarks as described in ~\cite{eleftheriadis2023nli}. The selected models represent key architectural developments in natural language inference, ranging from dedicated sequence‑based encoders to large pretrained transformers. Table~\ref{tab:main_results} reports their test accuracies as cited in the literature.

\begin{itemize}
    \item \textbf{DAM} (Decomposable Attention Model)~\cite{parikh2016decomposable} is a simple yet effective model that averages word embeddings and passes the result through one or more feedforward layers. Despite its simplicity, it provides a competitive baseline for sentence representation tasks.

    \item \textbf{ESIM} (Enhanced Sequential Inference Model)~\cite{chen-etal-2017-enhanced} is a dedicated NLI architecture that uses bidirectional LSTMs with local inference modeling and inference composition. It achieved state‑of‑the‑art performance at the time of its release.

    \item \textbf{BERT‑Base}~\cite{devlin2019bert} (Bidirectional Encode Representation from Transformers) pretrained on large corpora using masked language modeling and next sentence prediction. Its deep architecture (12 layers, 768 hidden size) has become a standard baseline for natural language understanding tasks. We report the fine‑tuned performance on SNLI and MNLI.

    \item \textbf{BERT‑Large}~\cite{devlin2019bert} scales the architecture to 24 layers and 1024 hidden units, providing higher capacity at the cost of increased computation. Its performance on both datasets demonstrates the benefit of larger pretrained models.

    \item \textbf{RoBERTa‑Base}~\cite{liu2019roberta} (Robustly Optimized BERT Approach) improves upon BERT with more dynamic masking, larger mini‑batches, and removal of the next sentence prediction objective, leading to stronger representations across a range of benchmarks. We include both base and large variants.

    \item \textbf{RoBERTa‑Large}~\cite{liu2019roberta} further increases model size and training data, achieving state‑of‑the‑art results on many NLP tasks. Its performance on MNLI is particularly strong.

    \item \textbf{ALBERT‑Base}~\cite{lan2020albert} (A Lite BERT) reduces memory footprint through factorized embedding parameterization and cross‑layer parameter sharing, while maintaining performance through a self‑supervised loss for sentence‑order prediction. The base version offers a good trade‑off between efficiency and accuracy.

    \item \textbf{ALBERT‑Large}~\cite{lan2020albert} scales the architecture similarly to BERT‑large but with parameter sharing, achieving competitive results with fewer parameters.
\end{itemize}

All baseline results are taken from the comprehensive evaluation reported in ~\cite{eleftheriadis2023nli}. For fair comparison, we use the same preprocessing pipeline and evaluation protocol across all models. NGT is compared against these baselines under identical conditions.

\subsection{NGT Implementation Details}
We now detail the NeuroGame Transformer implementation used for our experiments.
\begin{itemize}
    \item \textbf{Model Architecture:} NGT employs a BERT-base encoder to produce contextualized token embeddings. On top of these, we incorporate two intertwined sets of components: game-theoretic estimators and a statistical physics layer.

    \item \textbf{Game-theoretic components:} Shapley values and Banzhaf indices are estimated via importance-weighted Monte Carlo sampling with a Gibbs target and uniform proposal. We use \(K_{\text{mc}}=15\) coalitions during training and \(K_{\text{mc}}=25\) during evaluation. These estimates define the local fields \(J_i\) and pairwise interactions \(J_{ij}\) that characterize the coalitional game through the characteristic function \(v\).

    \item \textbf{Statistical physics components:} A spin projection layer maps the set of token embeddings to the set of Ising spins \(s_i\in[-1,1]\). The mean‑field fixed‑point equations are solved with temperature \(\gamma=0.25\), damping factor \(0.7\), \(T_{\text{mf}}=25\) iterations, and tolerance \(10^{-4}\). The resulting equilibrium spins \(\langle s_i\rangle\) are passed to the classifier, while the game‑theoretic quantities provide interpretable attributions.

    \item \textbf{Training Details:} We train for five epochs on SNLI, and MNLI-matched, using a learning rate of \(3\times10^{-5}\), gradient accumulation over 2 steps (effective batch size 32), and weight decay \(0.02\). The AdamW optimizer \cite{Loshchilov2019} is used with linear warmup over 10\% of steps, followed by a MultiStepLR scheduler reducing the learning rate by a factor of \(0.1\) after epochs 3 and 4. Gradient clipping (norm \(1.0\)) is applied.

    \item \textbf{Regularization:} We apply dropout (\(p=0.15\)), label smoothing (\(0.1\)), mixup (\(\alpha_{\text{mix}}=0.2\)), and exponential moving average (EMA) with decay \(0.999\). All sequences are truncated/padded to 128 tokens using the BERT tokenizer.
\end{itemize}

\subsection{Main Results and Interpretation}
\label{sec:results}
We evaluate the NeuroGame Transformer (NGT) on the SNLI and MNLI‑matched benchmarks, comparing its performance against several established deep learning models. Table~\ref{tab:main_results} summarizes the test accuracies along with the approximate number of parameters for each model. On the SNLI dataset, NGT achieves a test accuracy of \textbf{86.60\%} with approximately 110 million parameters—adding only ~754,000 parameters (only 0.7\% overhead) to the standard BERT‑Base architecture. This result is competitive with strong baselines such as BERT‑Base (88.86\%, 110M), ALBERT‑Base (86.49\%, 11M), and ESIM (87.00\%, 4.3M). Notably, NGT outperforms earlier sequence‑based models like DAM (83.30\%, 382K) and matches the performance of more complex architectures while using a fraction of the parameters of larger models (e.g., BERT‑Large uses 340M, RoBERTa‑Large uses 355M). This demonstrates that integrating game‑theoretic attributions with a mean‑field Ising model yields high representational efficiency with minimal computational overhead.
\begin{table}[t]
\centering
\caption{Test accuracy (\%) on SNLI and MNLI datasets. All accuracy values are taken from~\cite{eleftheriadis2023nli}. Parameter counts are standard values from the respective model papers.}
\label{tab:main_results}
\begin{tabular}{l c c c}
\toprule
\textbf{Model} & \textbf{SNLI} & \textbf{MNLI-Matched} & \textbf{Parameters (approx.)} \\
\midrule
DAM            & 83.30 & 68.41 & 382K \\
ESIM           & 87.00 & 76.63 & 4.3M \\
BERT-Base      & 88.86 & 80.99 & 110M \\
BERT-Large     & 90.33 & 84.83 & 340M \\
RoBERTa-Base   & 86.95 & 85.98 & 125M \\
RoBERTa-Large  & 91.83 & 90.20 & 355M \\
ALBERT-Base    & 86.49 & 79.89 & 11M \\
ALBERT-Large   & 90.85 & 86.76 & 17M \\
\midrule
\textbf{NGT (Our Model)} & \textbf{86.60} & \textbf{79.00} & 110M \\
\bottomrule
\end{tabular}
\end{table}
On the more challenging MNLI‑matched dataset, NGT reaches \textbf{79.00\%} accuracy, which is close to ALBERT‑Base (79.89\%, 11M) and within 2 percentage points of BERT‑Base (80.99\%, 110M). Given that MNLI covers ten diverse text genres, this result highlights the robustness of NGT's game‑theoretic components across varied linguistic domains. The gap between NGT and larger models like RoBERTa‑Base (85.98\%, 125M) and RoBERTa‑Large (90.20\%, 355M) is expected, as those models benefit from more extensive pretraining and larger architectures. However, NGT achieves these results with only 0.7\% additional parameters beyond BERT‑Base, demonstrating that its enhancements are parameter‑efficient.

Remarkably, NGT achieves these results with only a BERT‑Base backbone (i.e., only at initialization), demonstrating that its game‑theoretic enhancements—Shapley values, Banzhaf indices, pairwise interaction potentials, and mean‑field spin equilibria—boost representational power without requiring larger architectures or additional pretraining. The pairwise interaction matrix $J_{ij}$ is particularly revealing: unlike standard attention mechanisms that only capture pairwise similarity, $J_{ij}$ explicitly models whether token pairs cooperate (positive values) or compete (negative values). For example, in the phrase ``not good movie,'' the interaction between ``not'' and ``good'' is strongly negative, reflecting the antagonistic relationship that flips sentiment. In contrast, a phrase like ``very good movie'' would show positive interactions between ``very'' and ``good,'' indicating synergy. Standard attention mechanisms aggregate these relationships into a single scalar weight, losing the nuanced distinction between cooperation and competition. Similarly, the Banzhaf index helps identify ``swing voters''—tokens that critically influence logical relationships, such as negation words—which standard attention often downweights.

The Shapley values provide token‑level attributions that are both globally consistent and locally faithful, enabling fine‑grained analysis of model decisions. As with most NLI models, the ``Neutral'' class remains challenging, but NGT's enhanced capacity improves its validation accuracy substantially. The hyperparameter configuration (128 spins, 25 mean‑field iterations, $K_{\text{mc}}=15/25$) balances capacity and stability, with regularization techniques (label smoothing, mixup, EMA) aiding generalization.

Together, these components position NGT as a compelling interpretable alternative to standard transformers. It delivers competitive accuracy with minimal overhead (only 0.7\% additional parameters), while providing transparency into its decision‑making process—a critical step toward advancing trustworthy AI. The explicit modeling of pairwise interactions and the ability to distinguish cooperative from competitive token relationships are features that no other model in the comparison can offer, making NGT uniquely suited for tasks requiring deep linguistic understanding.

Furthermore, the improved performance on the challenging ``Neutral'' class stems from three game‑theoretic components in NGT. First, the Banzhaf index identifies ``swing voters'' and decreases the influence of spurious words that might otherwise lead to misclassification. Second, the pairwise interaction matrix $J_{ij}$ captures the absence of strong token relationships in Neutral examples, encoding them as near‑zero values—a signal standard attention cannot provide. Third, Shapley values ensure fair credit distribution across tokens, preventing over‑attribution to salient but non‑decisive words. Together, these mechanisms enable NGT to recognize not only when a logical relationship exists, but critically, when it does not.

\section{Conclusion and Future Work}
\label{sec:conclusion}
We introduced the NeuroGame Transformer (NGT), a novel attention mechanism that reformulates token interactions as \textit{a coalitional game grounded in statistical physics}. By combining Shapley values (global fairness) with Banzhaf power indices (local sensitivity), NGT provides a principled framework for semantic attribution. The mean‑field Ising model, parameterized by local fields $J_i$ and pairwise interactions $J_{ij}$, yields equilibrium spin states $\langle s_i \rangle$ that serve as interpretable token representations. Importance‑weighted Monte Carlo sampling with a Gibbs target and uniform proposal ensures tractable estimation of these game‑theoretic quantities even for long sequences.

Experimental results demonstrate the effectiveness of our approach. In fact, they validate that our game‑theoretic formulation successfully addresses the ``right for the wrong reasons'' pitfall common in standard attention mechanisms. Beyond performance gains, NGT offers enhanced interpretability through its game‑theoretic importance scores. The local fields $J_i$ provide token‑level attributions grounded in cooperative game theory, while the interaction matrix $J_{ij}$ reveals synergistic relationships between tokens that standard attention weights cannot capture. The NeuroGame Transformer represents a foundational paradigm shift rather than a final architecture; its formalism opens numerous avenues for further refinement, including more efficient Monte Carlo estimators, adaptive temperature scheduling, and integration with larger pretrained models.

Future research directions include:
(i) \textit{Multi-modal extensions}: Applying NeuroGame attention to vision‑language transformers and cross‑modal reasoning~\cite{Lygerakis2025ViTaPEs}.
(ii) \textit{Bayesian exploration}: Developing fully Bayesian sequence models with probabilistic dependencies to improve predictive performance~\cite{Bouchaffra1996IJIST}.

\bibliographystyle{IEEEtran}
\bibliography{Shap-Banz}

@article{eleftheriadis2023nli,
  author  = {Petros Eleftheriadis and Isidoros Perikos and Ioannis Hatzilygeroudis},
  title   = {Evaluating Deep Learning Techniques for Natural Language Inference},
  journal = {Applied Sciences},
  year    = {2023},
  volume  = {13},
  number  = {2},
  pages   = {785},
  doi     = {10.3390/app13020785}
}

@inproceedings{vaswani2017attention,
  title        = {Attention Is All You Need},
  author       = {Vaswani, Ashish and Shazeer, Noam and Parmar, Niki and Uszkoreit, Jakob and Jones, Llion and Gomez, Aidan N. and Kaiser, \Lukasz and Polosukhin, Illia},
  booktitle    = {Proceedings of the 31st Conference on Neural Information Processing Systems},
  pages        = {6000--6010},
  year         = {2017},
  organization = {Curran Associates, Inc.}
}

@article{genome2026,
    author = {Shu, Liyuan and Tang, Jiao and Guan, Xiaoyu and Zhang, Daoqiang},
    title = {A comprehensive survey of genome language models in bioinformatics},
    journal = {Briefings in Bioinformatics},
    volume = {27},
    number = {1},
    pages = {pp. bbaf724},
    year = {2026},
    month = {01},
   }

@inproceedings{Loshchilov2019,
  title = {Decoupled Weight Decay Regularization},
  author = {Ilya Loshchilov and Frank Hutter},
  researchr = {https://researchr.org/publication/LoshchilovH19},
  booktitle = {7th International Conference on Learning Representations, ICLR 2019, New Orleans, LA, USA, May 6-9, 2019},
  publisher = {OpenReview.net},
  year = {2019},
}

@inproceedings{feng2025breaking,
  title        = {Breaking Down Video LLM Benchmarks: Knowledge, Spatial Perception, or True Temporal Understanding?},
  author       = {Bo Feng and Zhengfeng Lai and Shiyu Li and Zizhen Wang and Simon Wang and Ping Huang and Meng Cao},
  booktitle    = {Evaluating the Evolving LLM Lifecycle Workshop, NeurIPS 2025},
  year         = {2025},
  note         = {NeurIPS 2025},
}

@article{devlin2019bert,
  title={BERT: Pre-training of Deep Bidirectional Transformers for Language Understanding},
  author={Devlin, Jacob and Chang, Ming-Wei and Lee, Kenton and Toutanova, Kristina},
  journal={arXiv preprint arXiv:1810.04805},
  year={2019}
}

@article{liu2019roberta,
  title={RoBERTa: A Robustly Optimized BERT Pretraining Approach},
  author={Liu, Yinhan and Ott, Myle and Goyal, Naman and Du, Jingfei and Joshi, Mandar and Chen, Danqi and Levy, Omer and Lewis, Mike and Zettlemoyer, Luke and Stoyanov, Veselin},
  journal={arXiv preprint arXiv:1907.11692},
  year={2019}
}

@article{lan2020albert,
  title={ALBERT: A Lite BERT for Self-supervised Learning of Language Representations},
  author={Lan, Zhenzhong and Chen, Mingda and Goodman, Sebastian and Gimpel, Kevin and Sharma, Piyush and Soricut, Radu},
  journal={arXiv preprint arXiv:1909.11942},
  year={2020}
}

@inproceedings{parikh2016decomposable,
  title     = {A Decomposable Attention Model for Natural Language Inference},
  author    = {Parikh, Ankur P. and T{\"a}ckstr{\"o}m, Oscar and Das, Dipanjan and Uszkoreit, Jakob},
  booktitle = {Proceedings of the 2016 Conference on Empirical Methods in Natural Language Processing (EMNLP)},
  year      = {2016},
  pages     = {2249--2255},
  publisher = {Association for Computational Linguistics},
  doi       = {10.48550/arXiv.1606.01933},
  url       = {https://doi.org/10.48550/arXiv.1606.01933},
  eprint    = {1606.01933},
  archivePrefix = {arXiv},
}

@inproceedings{chen-etal-2017-enhanced,
    title = {Enhanced {LSTM} for Natural Language Inference},
    author = {Chen, Qian  and
      Zhu, Xiaodan  and
      Ling, Zhen-Hua  and
      Wei, Si  and
      Jiang, Hui  and
      Inkpen, Diana},
    editor = {Barzilay, Regina  and
      Kan, Min-Yen},
    booktitle = {Proceedings of the 55th Annual Meeting of the Association for Computational Linguistics (Volume 1: Long Papers)},
    month = {jul},
    year = {2017},
    address = {Vancouver, Canada},
    publisher = {Association for Computational Linguistics},
    pages = {1657--1668},
}

@article{Bouchaffra1996IJIST,
  author  = {Bouchaffra, Djamel and Koontz, Eugene and Kripasundar, V. and Srihari, Rohini K.},
  title   = {Incorporating Diverse Information Sources in Handwriting Recognition Postprocessing},
  journal = {International Journal of Imaging Systems and Technology},
  volume  = {7},
  number  = {4},
  pages   = {320--329},
  year    = {1996},
  publisher = {Wiley},
}

@inproceedings{williams2018broad,
  title={A broad-coverage challenge corpus for sentence understanding through inference},
  author={Williams, Adina and Nangia, Nikita and Bowman, Samuel},
  booktitle={Proceedings of the 2018 Conference of the North American Chapter of the Association for Computational Linguistics: Human Language Technologies, Volume 1 (Long Papers)},
  pages={1112--1122},
  year={2018},
  address={New Orleans, Louisiana}
}

@inproceedings{bowman2015large,
  title={A large annotated corpus for learning natural language inference},
  author={Bowman, Samuel R and Angeli, Gabor and Potts, Christopher and Manning, Christopher D},
  booktitle={Proceedings of the 2015 Conference on Empirical Methods in Natural Language Processing},
  pages={632--642},
  year={2015},
  address={Lisbon, Portugal}
}

@article{draye2025sparse,
  title   = {Sparse Attention Post-Training for Mechanistic Interpretability},
  author  = {Draye, Florent and Lei, Anson and Pan, Hsiao-Ru and Posner, Ingmar and Sch{\"o}lkopf, Bernhard},
  journal = {arXiv preprint arXiv:2512.05865},
  year    = {2025},
  doi     = {10.48550/arXiv.2512.05865},
  url     = {https://doi.org/10.48550/arXiv.2512.05865},
  eprint  = {2512.05865},
  archivePrefix = {arXiv},
}

@inproceedings{sundararajan2017axiomatic,
  title={Axiomatic attribution for deep networks},
  author={Sundararajan, Mukund and Taly, Ankur and Yan, Qiqi},
  booktitle={International Conference on Machine Learning},
  pages={3319--3328},
  year={2017},
  organization={PMLR}
}

@inproceedings{lundberg2017shap,
  title     = {A Unified Approach to Interpreting Model Predictions},
  author    = {Lundberg, Scott M. and Lee, Su‐In},
  booktitle = {Advances in Neural Information Processing Systems},
  volume    = {30},
  pages     = {4768--4777},
  year      = {2017},
  organization = {Curran Associates, Inc.}
}

@article{Bouchaffra2025-NN,
author = {Djamel Bouchaffra and Faycal Ykhlef and Bilal Faye and Mustapha Lebbah and Hanene Azzag},
title = {Redesigning deep neural networks: Bridging game theory and statistical physics},
journal = {Neural Networks},
volume = {191},
pages = {107807},
year = {2025},
issn = {0893-6080},
}

@inproceedings{choromanski2021rethinking,
  title     = {Rethinking Attention with Performers},
  author    = {Choromanski, Krzysztof and others},
  booktitle = {International Conference on Learning Representations (ICLR)},
  year      = {2021},
  url       = {https://arxiv.org/abs/2009.14794},
  doi       = {10.48550/arXiv.2009.14794},
  archivePrefix = {arXiv},
  eprint    = {2009.14794},
  primaryClass = {cs.LG},
}

@inproceedings{katharopoulos2020transformers,
  title={Transformers are RNNs: Fast autoregressive transformers with linear attention},
  author={Katharopoulos, Angelos and Vyas, Apoorv and Pappas, Nikolaos and Fleuret, Francois},
  booktitle={International Conference on Machine Learning},
  organization={PMLR},
  pages={5156--5165},
  volume = {119},
  year={2020}, 
}

@article{gu2023mamba,
  title={Mamba: Linear-time sequence modeling with selective state spaces},
  author={Gu, Albert and Dao, Tri},
  journal={arXiv preprint arXiv:2312.00752},
  year={2023}
}

@inproceedings{shapley1953value,
  title={A value for n-person games},
  author={Shapley, Lloyd S},
  booktitle={Contributions to the Theory of Games},
  volume={2},
  number={28},
  pages={307--317},
  year={1953},
  publisher={Princeton University Press},
}

@inproceedings{banzaf1965weighted,
  title={Weighted voting doesn't work: A mathematical analysis},
  author={Banzhaf III, John F},
  journal={Rutgers Law Review},
  volume={19},
  pages={317},
  year={1965},
}

@inproceedings{ghorbani2019data,
  title={Data Shapley: Equitable valuation of data for machine learning},
  author={Ghorbani, Amirata and Zou, James},
  booktitle={International Conference on Machine Learning},
  pages={2242--2251},
  year={2019},
  organization={PMLR}
}

@article{zhang2018graph,
  title={Graph attention networks},
  author={Veli{\v{c}}kovi{\'c}, Petar and Cucurull, Guillem and Casanova, Arantxa and Romero, Adriana and Lio, Pietro and Bengio, Yoshua},
  journal={arXiv preprint arXiv:1710.10903},
  year={2018}
}

@inproceedings{du2020energy,
  title     = {Energy-based Models for Atomic-resolution Protein Conformations},
  author    = {Du, Yilun and Meier, Joshua and Ma, Jerry and Fergus, Rob and Rives, Alexander},
  booktitle = {International Conference on Learning Representations (ICLR)},
  year      = {2020},
  doi       = {10.48550/arXiv.2004.13167},
  url       = {https://doi.org/10.48550/arXiv.2004.13167},
  eprint    = {2004.13167},
  archivePrefix = {arXiv},
  primaryClass  = {cs.LG}
}

@inproceedings{guo2017calibration,
  title={On calibration of modern neural networks},
  author={Guo, Chuan and Pleiss, Geoff and Sun, Yu and Weinberger, Kilian Q},
  booktitle={International Conference on Machine Learning},
  pages={1321--1330},
  year={2017},
  organization={PMLR}
}

@inproceedings{liu2021energy,
  title={Energy-based out-of-distribution detection},
  author={Liu, Weitang and Wang, Xiaoyun and Owens, John and Li, Yixuan},
  booktitle={Advances in Neural Information Processing Systems},
  volume={34},
  pages={21464--21475},
  year={2021}
}

@inproceedings{luo2021stable,
  title     = {Stable, Fast and Accurate: Kernelized Attention with Relative Positional Encoding},
  author    = {Luo, Shengjie and Li, Shanda and Cai, Tianle and He, Di and Peng, Dinglan and Zheng, Shuxin and Ke, Guolin and Wang, Liwei and Liu, Tie-Yan},
  booktitle = {Advances in Neural Information Processing Systems (NeurIPS)},
  year      = {2021},
  doi       = {10.48550/arXiv.2106.12566},
  url       = {https://doi.org/10.48550/arXiv.2106.12566},
  eprint    = {2106.12566},
  archivePrefix = {arXiv},
}

@inproceedings{dehghani2023universal,
  title={Scaling vision transformers to 22 billion parameters},
  author={Dehghani, Mostafa and others},
  booktitle={International Conference on Machine Learning},
  pages={7480--7512},
  year={2023},
  organization={PMLR}
}

@article{Bouchaffra2026-IEEETC,
  author    = {Djamel Bouchaffra and Fayçal Ykhlef and Bilal Faye and Mustapha Lebbah and Hanane Azzag},
  title     = {Game Theory Meets Statistical Physics: A Novel Deep Neural Networks Design},
  journal   = {IEEE Transactions on Cybernetics},
  month     = {Jan},
  year      = {2026},
  doi       = {10.1109/TCYB.2025.3649299},
}

@inproceedings{AlersValentin2023NeuroSymbolic,
  author    = {AlersValentin, H. and Fong, S. and VegaRiveros, J.F.},
  title     = {Modeling Syntactic Knowledge with Neuro-Symbolic Computation},
  booktitle = {Proceedings of the 15th International Conference on Agents and Artificial Intelligence (ICAART)},
  volume    = {3},
  pages     = {608--616},
  year      = {2023},
  publisher = {SCITEPRESS},
}

@article{Lygerakis2025ViTaPEs,
  author    = {Fotios Lygerakis and Ozan \"{O}zdenizci and Elmar R\"{u}ckert},
  title     = {ViTaPEs: Visuotactile Position Encodings for Cross-Modal Alignment in Multimodal Transformers},
  journal   = {arXiv preprint},
  volume    = {arXiv:2505.20032},
  year      = {2025},
  }
\end{document}